\documentclass{article}

\usepackage{arxiv}

\usepackage[utf8]{inputenc} 
\usepackage[T1]{fontenc}    
\usepackage{hyperref}       
\usepackage{url}            
\usepackage{booktabs}       
\usepackage{amsfonts}       
\usepackage{nicefrac}       
\usepackage{microtype}      
\usepackage{lipsum}
\usepackage{graphicx}

\usepackage{amsmath}
\usepackage{amsthm}
\usepackage{bbm}
\usepackage{algorithm}
\usepackage{algorithmic}
\usepackage{multirow}
\usepackage{caption}
\usepackage{subcaption}
\newtheorem{theorem}{Theorem}

\title{A Unified Framework for Online Trip Destination Prediction}

\author{Victor Eberstein  \\
                Chalmers University of Technology\\
                SE-412 96 Gothenburg, Sweden\\
                \texttt{vicebb@chalmers.se}\\
           \And
            Jonas Sjöblom \\
                Chalmers University of Technology\\
                SE-412 96 Gothenburg, Sweden\\
                \texttt{jonas.sjoblom@chalmers.se}\\
           \And
                Nikolce Murgovski  \\
                Chalmers University of Technology\\
                SE-412 96 Gothenburg, Sweden\\
                \texttt{nikolce.murgovski@chalmers.se}\\
           \And
                Morteza Haghir Chehreghani  \\
                Chalmers University of Technology\\
                SE-412 96 Gothenburg, Sweden\\
                \texttt{morteza.chehreghani@chalmers.se}\\
}

\begin{document}

\maketitle

\begin{abstract}
Trip destination prediction is an area of increasing importance in many applications such as trip planning, autonomous driving and electric vehicles.
Even though this problem could be naturally addressed in an online learning paradigm where data is arriving in a sequential fashion, the majority of research has rather considered the offline setting.
In this paper, we present a unified framework for trip destination prediction in an online setting, which is suitable for both online training and online prediction.
For this purpose, we develop two clustering algorithms and integrate them within two online prediction models for this problem.

We investigate the different configurations of clustering algorithms and prediction models on a real-world dataset.
We demonstrate that both the clustering and the entire framework yield consistent results compared to the offline setting.
Finally, we propose a novel regret metric for evaluating the entire online framework in comparison to its offline counterpart.
This metric makes it possible to relate the source of erroneous predictions to either the clustering or the prediction model.
Using this metric, we show that the proposed methods converge to a probability distribution resembling the true underlying distribution with a lower regret than all of the baselines.

\end{abstract}

\section{Introduction}
In today's society almost all newly produced cars are equipped with some sort of built-in navigation system using GPS-based data.
Consequently, new research areas have emerged in data analytics for transport systems due to the large amount of geospatial data being shared over cellular networks.
One such area is future route and destination prediction, which has been considerably focused on in the last decade.
Being able to accurately predict the future location and/or route of a vehicle has some obvious advantages for individual vehicles, e.g., avoiding traffic congestion, estimating travel time and electric range, improved personalization and adaptation, etc.
For the same reason, it can also be beneficial for the Traffic Management Centers (TMC) to better estimate future traffic situations.

Another potential application area for future route and destination prediction is in designing energy management systems for hybrid vehicles, i.e., to control the power split between the internal combustion engine and the electrical machine.
Optimal control of hybrid vehicles is a non-trivial task when a trip (or sequence of trips) is exceeding the all-electrical range.
Several studies have been performed comparing the simple electric vehicle/charge sustaining (EV/CS) strategy with blended strategies, i.e., strategies that continuously blend the energy from the battery and the fuel in such a way that the battery is depleted at the very end of a trip.
In comparison to EV/CS, such strategies have been shown to yield a significant reduction of
the fuel consumption, e.g., up to 20\% as reported in \cite{KUM2010258}.
The main drawback of blended strategies is that they require a priori information about the trip in order to find the best possible strategy, e.g., destination, route, travel time, etc.
However, such information is not necessarily available.
The authors in \cite{6314910} noted that if a trip or route is recognized from a driving history, then one can employ a blended strategy for that specific trip.
If that is not possible for the specific trip, one can always revert to the simple EV/CS strategy.
The model they presented was able to reduce the fuel consumption by 1.5\%, without any a priori information in comparison to the EV/CS strategy.
In \cite{8103917}, the authors use driving histories in order to optimize the EMS and  learn a look-up table for the controller parameters based on frequently traversed routes.
The result is an EMS that consumes only 2.5\% more energy than the corresponding posterior global optimal result.

Thereby, in this paper, we study prediction of trip destination in different settings.
We denote the full trip history of an individual user by $X_u$.
An individual trip  will be referred to as $x_u(i)$, for $i = 1,\ldots, N_u$, where $N_u$ is the length of the trip history, i.e. the number of trips available for that user.
Each trip consists of GPS-coordinates, i.e. the latitude and longitude pair of the source and destination.
Let $X_u(i) = \{x_u^s(i), x_u^d(i)\}$, where $x_u^s(i)$ and $x_u^d(i)$ represent the source and the destination of trip $x_u(i)$.
One can also form $X_u^s$ and $X_u^d$, which are the concatenation of all sources and destination in the entire trip history of user $u$.
In its simplest form, the task can then be specified to predict $x_u^d(i)$ from $x_u^s(i)$ for future trips $x_u(i)$.

Trip destination prediction can be considered in an online or an offline setting, which are the two very different ways of approaching the task.
In the offline setting, one would train a model on the full trip history $X_u^s$ and $X_u^d$ in order to predict all future trips.
However, in the online setting one does not assume to have access to the full trip history, and instead considers the trips arriving in a sequential or incremental fashion.
In other words, the prediction of $x_u^d(i)$ from $x_u^s(i)$ is made after having observed all trips $X_u(i')$ for $i' < i$ and once the actual trip $X_u(i)$ is observed the model is immediately updated.
The problem in this setting then becomes to perform as good as possible in comparison to the offline model trained on the full trip history.

As will be discussed in Section 2, several previous works study trip destination problem from different aspects. However, those studies separate the prediction model from the formation of the prediction space, i.e., the clustering of candidate locations. In particular, the clustering is not considered in the evaluation of the final model. In addition, almost all of the previous works assume an offline setting where they  investigate the  model after it has been trained on a full dataset.

In this paper, we develop and investigate a unified framework fully suitable for online training and prediction.
For this purpose, two novel online clustering algorithms are used with two different online prediction models, which enable the entire framework to be investigated in an online setting.
The two clustering algorithms are adaptations of an incremental variant of the DBSCAN clustering algorithm  \cite{ester1996density}.
Instead of storing all previously seen points, the two proposed algorithms store centroids belonging to the different clusters as well as the outlier points.
The first prediction model is a probabilistic Bayesian model, using sequential updates for its parameters.
The second prediction model is an adaptation of an expert model, where the set of available experts is dynamic.
Both models yield a distribution over the possible destinations, which can be compared to the true distribution obtained by the offline model.

The different configurations of clustering algorithm and prediction model are investigated on a real-world dataset.
At first, the clustering algorithms are evaluated using supervised clustering metrics, where the offline DBSCAN is considered the true clustering.
Secondly, the entire framework is evaluated using accuracy, which is traditionally used for trip destination prediction.
It is shown that these configurations yield consistent results to the offline model on unseen data.
Furthermore, the online framework is evaluated using a novel metric based on the Hellinger distance, such that it is possible to relate the source of erroneous predictions to either the clustering or the prediction model.
From this, one can observe that the learning improves as more trips are added.

The rest of the paper is organized as follows.
In Section 2, we review the related works and position our contribution w.r.t. them.
In Section 3, we introduce and formalize an offline methodology to serve as the baseline for the proposed online models.
It consists of clustering candidate locations and estimating transition probabilities between the discovered locations.
Section 4 extends and adapts both the clustering and the prediction model in the offline methodology to an online setting, i.e., the case where data arrives sequentially over time.
In addition, we introduce an evaluation metric to measure the performance of the entire pipeline.
In Section 5, we conduct the experiments and evaluate our proposed models on  real-world datasets.
Finally, in Section 6, we conclude the paper and discuss the future directions.

\section{Related work}
To the best of our knowledge, \cite{ashbrook2003using} is one of the first works on the topic of predicting user destination from historical GPS-data.
This model can be split into two parts:
(i) clustering the raw GPS-points into candidate destination, and
(ii) using a Markov model with the candidate locations as states to predict the next destination.
This methodology, wherein the candidate locations are first found from historical GPS-logs and later used in a graphical model (some variants of Markov models) has since been adopted by a number of subsequent works \cite{alvarez2010trip,simmons2006learning,panahandeh2017driver,zong2019trip}.

In \cite{simmons2006learning}, a Hidden Markov Model (HMM) is built using the links obtained by mapping GPS positions to a map database.
For an on-going trip, this model can be used to find both the next road link and the next sequence of road links that are most likely.
In other words, it can be used to predict the most likely future route/destination of an on-going trip.
In order to make the model independent of a map database, the authors in \cite{alvarez2010trip} suggested to consider support points along traversed routes, e.g., intersections that could differentiate between possible destinations.
Using the support points as observable states and the candidate locations as hidden states, an HMM is built and used to predict the destination of trips.
Other approaches that have been investigated include similarity matching between the current trip and previously stored trips.
In \cite{laasonen2005route}, the authors suggest to predict the route by applying string matching techniques to a database of stored routes.
The authors in \cite{froehlich2008route} present a similarity algorithm to cluster similar trips using hierarchical clustering and then predict the most likely route/destination.

A more recent work is \cite{panahandeh2017driver}, wherein a probabilistic Bayesian model conditioned on the origin and current road link is used to predict the future destination.
In \cite{epperlein2018bayesian} the authors develop a Bayesian framework to model route patterns, and present a model based on Markov chains to probabilistically predict the route/destination of an ongoing trip.
Perhaps the closest related work is \cite{filev2011contextual}, wherein a k-nearest neighbors (kNN) model is used to find the most important destinations and a Markov chain model is employed for predictions.
Both the available destinations and the Markov chain transitions are updated in an online setting.
However, they are not considering the clustering method as part of their online framework.
Subsequently, their model is only evaluated using accuracy on repeat trips, i.e., trips that are never repeated are removed in a preproccesing step. They do not use an explicit destination clustering method,  their online model is simply based on counting, and in their evaluations they use only accuracy.

There are also a significant number of works using slightly different types of data.
In \cite{asehuman12}, the authors look at two online learning algorithms applied to a Bayes net model to predict destinations of users.
However, the data that they consider is based on voluntary check-ins from social networking websites.
There has also been extensive work on public transport and taxi data, e.g., in 2015 the method  \cite{taxi} won the \textit{ECML/PKDD 15: Taxi Trajectory Prediction} on Kaggle.
The proposed model consists of a clustering model to find candidate locations, followed by a recurrent neural network (since it was trained on initial trajectories of trips).
The destinations were predicted using a weighted average of the softmax output of the different cluster centroids.
Taxi services have also been phrased as a reinforcement problem, see for example \cite{optimizetaxi}, where the authors use the total profit of a taxi driver as the reward function, the location and status of the taxi as the state space, and the choices of operating the taxi as the action space.
Finally, the authors in \cite{chen} propose an offline trip destination model for public transport based on trip histories in an evaluation set and the neighboring user trips.

We also note that there are problem formulations that are similar, or closely related, to trip destination prediction.
One such example is trip purpose prediction, see, e.g. \cite{ERMAGUN201796,CHEN2010830,XIAO2016447}, where one tries to predict the purpose of a trip, for instance shopping, restaurant, education, etc.
Another example is when trip destination prediction is used as a sub-problem, e.g., the authors in \cite{gathering} use trip destination prediction of taxi data in order to forecast gathering events.
We also note that destination prediction in general can be employed by trip planning methods, e.g., the methods described in \cite{niklas,niklas2} for navigation and in \cite{onlineMinimax} for bottleneck identification, which propose effective plans for given sources and destinations.

All of these works assume that the prediction model is separated from clustering candidate locations which represent the prediction space.
Thus, they consider the clustering as a preprocessing step, and then narrow down the trips that are being considered to only the frequent trips.
Even in the few cases where the clustering has been mentioned and explored explicitly, it has not been reflected in the evaluation of the final model.
Furthermore, almost all of the previous works assume an offline setting for their model, fully investigating their model after it has been trained on a full dataset. In this paper,  we develop  a unified online learning  framework  that i) takes into account the creation and evolution of clusters in a consistent way, and ii) learns the model and the parameters in an online fashion.
For this purpose, we propose  two novel online clustering algorithms to be  used with two different online prediction models, and investigate the entire framework in  a fully  online setting.

\section{Offline trip prediction}
In this section, we introduce an offline mode for prediction of trip destinations.
The offline model will serve as the baseline, to which the proposed online model will be compared.
Our model is adopted from the methodology proposed by \cite{ashbrook2003using}, which suggests to first use clustering to find candidate locations and then estimate the transition probabilities between each of the found locations.

\subsection{Clustering}
For clustering, we use DBSCAN \cite{ester1996density} which is a simple yet effective density based method.
It does not require to fix the number of clusters in advance and is also robust to outliers.
It has also been shown to work well on low dimensional data.
Since the data being considered in this study is 2-dimensional with latitude and longitude features, DBSCAN is considered to be appropriate.
Furthermore, this method also has the advantage of making the clusters easy to interpret in terms of the choice of parameter values.
The algorithm takes two parameters, $\epsilon$ which is the minimum distance between the points to be considered inside the same neighborhood, and $m$ which is the minimum number of points required inside a neighborhood to form a cluster.
Since the trip history $X_u$ consists of GPS-coordinates, a suitable distance metric is the Haversine distance.
The Haversine distance \cite{robusto1957cosine} between two points, $x$ and $y$, consisting of latitude and longitude features is defined as
\begin{align*}
        \textsc{distance}(x,y) =
         2\arcsin\left(\sqrt{f_1( x_{\text{lat}},y_{\text{lat}})+f_2(x_{\text{long}}, y_{\text{long}})}\right).
\end{align*}
Here we have used the short-hand notation:
\begin{align*}
    f_1(x,y) = \sin^2\left(\frac{x-y}{2}\right), \quad
    f_2(x,y) = \cos(x)\cos(y)\sin^2\left(\frac{x-y}{2}\right).
\end{align*}

Ideally, it should not matter whether one chooses to cluster the sources $X_u^s$ or $X_u^d$, since the end of one trip corresponds to the beginning of another.
However, one recurring problem when working with GPS-data is the initial time it takes for the GPS receiver to acquire the satellite signal.
This delay is usually in the range between 10 to 60 seconds, but could be as long as a couple of minutes, which makes the source of every trip more uncertain than the destination.
Thus, it is reasonable to use $X_u^d$ to form the clusters and find the candidate locations.

Clustering all destinations in $X_u^d$ using DBSCAN returns the cluster labels $C_u^d$ for the destinations, where $c_u^d(i)$ for $i = 1\ldots, N_u$ is used to denote the label of an individual trip destination.
Further, let the set of all cluster labels be denoted by $M_u = \{0, \ldots, K-1\}$, i.e. $c_u^d(i) \in M_u\cup\{-1\}$, where  $c_u^d(i) = -1$ corresponds to an outlier.
One still needs to find the corresponding cluster labels of the source of each trip, i.e. $C_u^s$, even though they were not used in the clustering procedure.
For the reasons already mentioned (robustness of the source), one cannot simply assign the source to a cluster if the distance to a dense neighborhood is less than $\epsilon$.
Instead, whenever there are two or more clusters, we compute $C_u^s$ according to Alg. \ref{alg:source_cluster}, which essentially assigns a cluster label to $c_u^s(i)$ if the closest cluster is $\delta$ times closer than the second closest cluster, and otherwise it is considered an outlier.

\begin{algorithm}
\caption{Find the corresponding cluster label $c_u^s(i)$ of the source of a trip $x_u^s(i)$.}
\textbf{Input: }{Set of cluster labels $M_u$, source $x_u^s(i)$} \\
\textbf{Parameters: }{distance threshold $\delta$}
    \begin{algorithmic}
        \label{alg:source_cluster}
        \STATE Find two closest clusters $c_1, c_2 \in M_u$ with distances $d_1 < d_2$ from $x_u^s(i)$
        \IF{$d_2/d_1 > \delta$}
            \STATE return $c_1$
        \ELSE
            \STATE return $-1$
        \ENDIF
    \end{algorithmic}
\end{algorithm}

\subsection{Bayesian prediction model}
After the clustering procedure, the entire trip history can be represented with the cluster labels $C_u^s$ and $C_u^d$ for the sources and destinations respectively.
All the transitions made by the user $u$ are $c_u^s(i) \to c_u^d(i)$ for $i = 1,\ldots, N_u$, where $c_u^s(i), c_u^d(i) \in M_u\cup\{-1\}$.

First of all, consider the user-specific distribution of the destination $p(c_u^d)$, where $p(c_u^d) = k$ represents the probability of the destination being $k \in M_u\cup\{-1\}$.
Next, consider the distributions when the destination is conditioned on the source, i.e. $p(c_u^d | c_u^s)$.
Now, $p(c_u^d = k | c_u^s = j)$ with $j,k \in M_u\cup\{-1\}$, represents the probability of the destination being $k$ given that the trip starts at $j$.

We assume that the distribution $p(c_u^d)$ and all of the conditional distributions $p(c_u^d | c_u^s)$ follow categorical distributions.
That is, $p(c_u^d) \sim \text{Cat}(K+1, \Lambda)$ with event probabilities $\Lambda = (\Lambda_{-1},\Lambda_{0}, \ldots, \Lambda_{K-1})$, and  $p(c_u^d | c_u^s = j) \sim \text{Cat}(K, \lambda_j)$ with event probabilities $\lambda_j = (\lambda_{j,-1}, \lambda_{j,0}, \ldots, \lambda_{j,K-1})$.
The event probabilities are interpreted as $\Lambda_k = p(c_u^d = k)$ for $p(c_u^d)$, as well as $\lambda_{jk} = p(c_u^d = k | c_u^s = j)$ for $p(c_u^d | c_u^s)$.
In the following subsection, we describe how to estimate the parameters of $\Lambda$ and $\lambda_j$ for $j=-1, 0\ldots, K-1$.

\paragraph{Parameter estimation.}
We adopt a Bayesian approach to estimate the parameters.
In this approach, we use a prior distribution over $\Lambda$ and $\lambda_j$ for all $j=-1,0,\ldots,K-1$, after which Bayes' rule can be used to update the posterior distribution.
Starting with $\Lambda$, assume that $p(\Lambda | \beta)$ follows a Dirichlet distribution with some concentration parameters $\beta = (\beta_{-1},\beta_{0}, \ldots,  \beta_{K-1})$.
Using Bayes' rule, the posterior $p(\Lambda | C_u^d, \beta)$ can then be computed as:
\begin{align*}
    p(\Lambda | C_u^d, \beta)= \frac{p(C_u^d | \Lambda)p(\Lambda | \beta)}{p(C_u^d | \beta)},
\end{align*}
where $\beta$'s are the hyperparameters.
The two terms in the numerator, i.e. $p(C_u^d | \Lambda)$ and $p(\Lambda | \beta)$, are often referred to as the likelihood and the prior distribution respectively.
On the other hand, the denominator $p(C_u^d | \beta)$ is the marginal likelihood, or evidence, which refers to the distribution once the parameter $\Lambda$ has been marginalized out.

Using the fact that the Dirichlet distribution is the conjugate prior to the categorical distribution, it holds that $p(\Lambda | C_u^d, \beta) \sim \text{Dir}(K+1, n + \beta)$, where $n = (n_{-1}, n_{0}, \ldots, n_{K-1})$ and $n_j$ is the number of trips ending in $j$.
Essentially, the hyperparameters $\beta$ can be treated as pseudocounts in the model.
In other words, the event probabilities $\Lambda$ are set to
\begin{align*}
    \Lambda_{j} = \frac{n_{j} + \beta_{j}}{\sum_j n_{j} + \beta_{j}}
\end{align*}
for $j=-1,0,\ldots, K-1$.

Estimating $\lambda_j$ for $j=-1,0\ldots, K-1$ follows a similar procedure as when estimating $\Lambda$.
Given a single value of $j$, assume that $p(\lambda_j | \alpha_j)$ follows a Dirichlet distribution with hyperparameters $\alpha_j = (\alpha_{j,0},\ldots, \alpha_{j,K-1})$.
According to Bayes' rule, the posterior $p(\lambda_j |C_u^{\{j\}}, \alpha_j)$ is determined by
\begin{align*}
    p(\lambda_j |C_u^{\{j\}}, \alpha_j)= \frac{p(C_u^{\{j\}} | \lambda_j)p(\lambda_j | \alpha_j)}{p(C_u^{\{j\}} | \alpha_j)},
\end{align*}
where $C_u^{\{j\}}$ is used to denote all trips starting in $j$.
Using conjugacy, it holds that $p(\lambda_j |C_u^{\{j\}}, \alpha_j) \sim \text{Dir}(K+1, \hat{n_j} + \alpha_i)$, where $\hat{n_j} = (n_{j,-1},n_{j,0}, \ldots, n_{j,K-1})$ and $n_{jk}$ is the number of trips going from $j$ to $k$.
Therefore, the event probabilities are
\begin{align*}
    \lambda_{jk} = \frac{n_{jk} + \alpha_{jk}}{\sum_j n_{jk} + \alpha_{jk}},
\end{align*}
and once again $\alpha_{jk}$ can be interpreted as pseudocounts.

We note that setting the hyperparameters to zero in a distribution will yield the maximum likelihood estimate of the corresponding event probabilities.
However, this would imply that all the transitions not present in the dataset will have a zero probability of occurring in the model, e.g. $\alpha_{jk} = 0$ and $n_{jk} = 0$ would result in $\lambda_{jk} =  p(c_u^d = k | c_u^s = j) = 0$.

\section{Online trip prediction}
In this section, we extend the offline prediction model to an online setting.
The trip history of a user, $X_u$, will now arrive sequentially, one trip at a time, which the offline model cannot handle.
The offline clustering requires the entire $X_u$'s to find the candidate location, and the prediction model requires the cluster labels $C_u$ to estimate the transition probabilities.
Thus, to make the entire pipeline online, both the clustering and the prediction model have to be adapted appropriately.

\begin{figure}[ht]
    \centering
    \includegraphics[width=1\columnwidth]{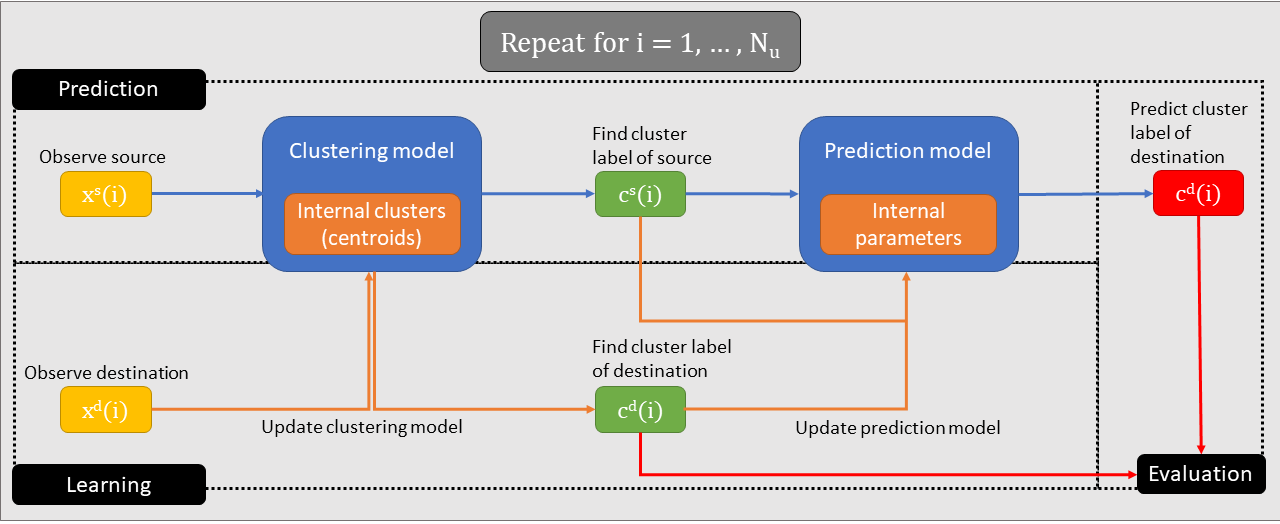}
    \caption{An overview of the proposed online framework, which includes a clustering and a prediction model, both updated incrementally. It is organized in three components: i) prediction, ii) learning, and iii) evaluation.}
    \label{fig:overview}
\end{figure}

The proposed online framework is shown in Fig. \ref{fig:overview}.
For each new trip, we use the clustering model to find the cluster of the source (starting point) of the trip.
The output from the clustering is then used to predict the destination using the prediction model.
Next, the actual destination is observed after the trip takes place and is used to update both the clustering and the prediction model (and their parameters).
Finally, the predicted destination is evaluated in comparison to the actual destination. Therefore, the framework consists of three main components: i) prediction, ii) learning, and iii) evaluation.
These steps are then repeated for every new trip that is observed.

\subsection{Online clustering}
Online variants of different clustering algorithms have already been proposed in the literature, e.g. an incremental variant of DBSCAN is proposed in \cite{ester1998incremental}.
Inspired by this incremental adaptation of DBSCAN, we propose two different variants of a DBSCAN to cluster the points online.
The main difference is that instead of storing the core points, these variants store core centroids and keep track of the number of points within a specified radius.
Both variants take the same parameters as the original DBSCAN algorithm, i.e. $m$ and $\epsilon$, representing the minimum number of points in a cluster and the minimum distance threshold respectively.

For every new point $x_u^d(i)$ that arrives, the point is clustered and $c_u^d(i)$ is obtained as a result of the clustering.
In order to determine $c_u^s(i)$, we employ the same approach as in the offline setting.
That is, in order to assign a label, the closest cluster has to be at least $\delta$ times closer than the second closest one, and otherwise it is considered an outlier.

\subsubsection{Online DBSCAN 1}
The first variant presented in Alg. \ref{alg2} takes an additional parameter $r$, which is used to determine the radii of the centroids that are stored.
When $r \to 0$, the centroids will naturally become points, in which case the method will behave as the incremental adaptation of DBSCAN in \cite{ester1998incremental}.
The centroids are stored as $(c_q, n_q, l_q)$, where the elements are the centroid itself, the number of points it contains, and the cluster label of the centroid respectively.
All non-clustered points are stored as $(x_s, n_s, t_s)$, where the elements are the point itself, the number of neighbors, and the timestamp of the point respectively.

\begin{algorithm}
\caption{Online clustering - V1}
\textbf{Input: }{New point $x$, and timestamp $t$} \\
\textbf{Parameters: }{distance threshold $\epsilon$, minimum number of points in a cluster $m$, radii fraction $r$ (i.e. $r\cdot\epsilon$ is the radii of the centroids)} \\
\textbf{Stored: }{non-clustered points $P = \{(x_s, n_s, t_s)\}_{s=1}^S$, \\ centroids $C = \{(c_q, n_q, l_q)\}_{q=1}^{Q}$}
\begin{algorithmic}[1]
\label{alg2}
\STATE $P_x = \{(x_s, n_s, t_s) \in P \ |\ \textsc{Distance}(x, x_s) <  \epsilon \}$
\FOR{$(x_s, n_s, t_s) \in P_x$}
    \STATE $n_s \leftarrow n_s + 1$
\ENDFOR
\STATE $q^* = \underset{q}{\arg\min} \  \textsc{Distance}(x, c_q)$
\IF{$\textsc{Distance}(x, c_{q^*}) <  r\cdot \epsilon$}
    \STATE $n_{q^*} \leftarrow n_{q^*} + 1$
\ELSE
    \STATE $N_x = \{n_q \ |\ (c_q, n_q, \cdot) \in C, \textsc{Distance}(x, c_q) <  (r+1)\cdot \epsilon \}$
    \STATE $n_x = |P_x|+\sum_{n_q \in N_x} n_q$
    \STATE $p_x = (x,n_x,t)$
    \STATE $P \leftarrow P\cup\{p_x\}$
\ENDIF
\STATE $\textsc{CheckForNewCentroids}(P_x\cup\{p_x\})$
\STATE $\textsc{DeleteOldPoints}(t)$
\end{algorithmic}
\end{algorithm}

Whenever a new point arrives, at first, we check for neighboring points and add the new point as a neighbor to the existing points (line 1-4).
Now, if the closest centroid contains the point, then the new point is simply added to the existing centroid (line 5-7).
Otherwise, the number of neighbors of the new point is computed, after which the new point is added to the non-clustered points (line 9-12).

Next, the function $\textsc{CheckForNewCentroids}(P_x\cup\{p_x\})$ looks at all non-clustered points in $P_x$ as well as the possible new point $p_x$ and finds the points that should be upgraded to a centroid (line 14).
This is the case for all points where $n_s \geq m-1$.
Depending on whether the neighbors are part of an existing centroid or not, three different scenarios can occur:
\begin{enumerate}
    \item Only neighbors in $P$: Add the centroid to a new cluster.
    \item Neighbors in a single cluster: Add the centroid to the existing cluster.
    \item Neighbors in multiple clusters: Merge the existing clusters and add the centroid to the resultant cluster.
\end{enumerate}
Finally, all the points that are too old according to the function $\textsc{DeleteOldPoints}(t)$ are removed from $P$ (line 15). This function is elaborated further in the experiments section.

\subsubsection{Online DBSCAN 2}
The second variant, shown in Alg. \ref{alg3}, does not have any additional parameters and the stored elements are slightly different.
All non-clustered points are still stored as $(x_s, n_s, t_s)$.
However, the centroids are now stored as $(c_k, n_k, \epsilon_k)$, where the first two elements are still the centroid and the number of points it contains, but the third element is now the centroid radius.
In this variant, there will be only a single centroid for each cluster, but the centroid will continuously grow when new points are encountered, i.e. the individual centroid radius will increase.
In this way, this variant is rather similar to the K-means clustering.

\begin{algorithm}
\caption{Online clustering - V2}
\textbf{Input: }{New point $x$, and timestamp $t$} \\
\textbf{Parameters: }{distance threshold $\epsilon$, minimum number of points in a cluster $m$} \\
\textbf{Stored: }{non-clustered points $P = \{(x_s, n_s, t_s)\}_{s=1}^S$, \\ centroids $C = \{(c_k, n_k, \epsilon_k)\}_{k=1}^{K}$}
\begin{algorithmic}[1]
\label{alg3}
\STATE $P_x = \{(x_s, n_s, t_s) \in P \ |\ \textsc{Distance}(x, x_s) <  \epsilon \}$
\FOR{$(x_s, n_s, t_s) \in P_x$}
    \STATE $n_s \leftarrow n_s + 1$
\ENDFOR
\STATE $k^* = \underset{k}{\arg\min} \  \textsc{Distance}(x, c_k) - \epsilon_k$
\IF{$\textsc{Distance}(x, c_{k^*}) - \epsilon_{k^*} <  0$}
    \STATE $n_{k^*} \leftarrow n_{k^*} + 1$
\ELSE
    \STATE $N_x = \{n_k \ |\ (c_k, n_k, \cdot) \in C, \textsc{Distance}(x, c_k)  - \epsilon_{k^*} <  \epsilon \}$
    \STATE $n_x = |P_x|+\sum_{n_k \in N_x} n_x$
    \STATE $p_x = (x,n_x,t)$
    \STATE $P \leftarrow P\cup\{p_x\}$
\ENDIF
\STATE $\textsc{UpdateClusters}(P_x\cup\{p_x\})$
\STATE $\textsc{DeleteOldPoints}(t)$
\end{algorithmic}
\end{algorithm}

Many of the steps of this algorithm are identical to Alg. \ref{alg2}.
The main difference is the individual radius of each centroid, $\epsilon_k$, which makes a difference when computing the closest cluster (line 5).
The only additional change is that the function $\textsc{CheckForNewCentroids}$ has been replaced by $\textsc{UpdateClusters}$ (line 14).
This function plays a similar role and looks at all non-clustered points in $P_x$ and finds the points that can be used to update the centroids, i.e. those with $n_s \geq m-1$.
Once again, depending on whether the neighbors are part of an existing centroid or not, three different scenarios can occur:
\begin{enumerate}
    \item Only neighbors in $P$: Create a new cluster.
    \item Neighbors in a single centroid: Update the existing centroid, i.e. the radius $\epsilon_k$, to contain the point.
    \item Neighbors in multiple centroids: Merge the existing centroids, i.e. update the centroid $c_k$ and radius $\epsilon_k$ to cover all of the previous centroids as well as the new point.
\end{enumerate}

\subsection{Bayesian model}
We adapt the offline prediction model to the online setting.
The difference is that the set of all cluster labels now can change with every new trip that is observed, i.e. $M_u$ has to be replaced with $M_u(i)$.
Here, $M_u(i)$ represents the set of all cluster labels at timestep $i$, i.e. after trip $x_u(i)$ as well as all previous trips have been observed.
This renders the distributions to be time dependent and dynamic.

Consider the distribution of the destination, $p_i(c_u^d)$, where $p_i(c_u^d = k)$ represents the probability of the destination being $k \in M_u(i)\cup\{-1\}$ at timestep $i$.
Assume it follows a categorical distribution, i.e. $p_i(c_u^d) \sim \text{Cat}(K(i)+1, \Lambda(i))$, where both the event probabilities $\Lambda(i)$ and the number of possible cluster labels $K(i)$ depend on $i$ in this setting.
These event probabilities are still interpreted as $\Lambda_k(i) = p_i(c_u^d = k)$.

Similarly, conditioning the destination on the source, $p_i(c_u^d|c_u^s)$, where $p_i(c_u^d = k|c_u^s=j)$ with $j,k\in M_u(i)\cup\{-1\}$ represents the probability of the destination being $k$ given that the trip starts at $j$ at timestep $i$.
Assume that this follows a categorical distribution, i.e. $p_i(c_u^d|c_u^s) \sim \text{Cat}(K(i)+1, \lambda_j(i))$, with the number of possible clusters $K(i)$ and event probabilities $\lambda_j(i)$.
The event probabilities are once again interpreted as $\lambda_{jk}(i) = p_i(c_u^d=k|c_u^s=j)$.

\paragraph{Parameter estimation.}
In this setting, we can still update the posterior distribution of $\Lambda(i)$ and $\lambda_j(i)$ using Bayes' rule.
However, here we employ sequential Bayesian updating, which works by letting the prior distribution in each timestep be the posterior distribution of the previous timestep.
In detail, looking at $\Lambda(i)$, the update rule can be written as
\begin{align*}
    p(\Lambda(i+1)\ |\ C_u^d[i+1], \beta) = \\ \frac{p_i(c_u^d(i+1)\ |\ \Lambda(i), \beta, C_u^d[i])p(\Lambda(i)\ |\ C_u^d[i], \beta)}{p(c_u^d(i+1)\ |\ \beta, C_u^d[i])},
\end{align*}
where $[i]$ is used to denote trip $i$ and all previous trips, whereas $(i)$ is used to denote the specific trip $i$.

There is one problem with this formulation, however, since $\Lambda(i+1)$ and $\Lambda(i)$ do not necessarily have the same number of classes/clusters.
This means that the prior $p(\Lambda(i)\ |\ C_u^d[i], \beta)$ needs to be changed to the following
\begin{align*}
    \hat{p}(\Lambda(i)\ |\ C_u^d[i], \beta) \sim \text{Dir}(K(i+1)+1, \hat{\Lambda}(i)),
\end{align*}
where $\hat{\Lambda}(i)$ should be computed from $\Lambda(i)$.
This is feasible, since the clustering algorithms in Alg. \ref{alg2} and Alg. \ref{alg3} return the updates from the clustering when creating new centroids.
This leads to the following scenarios for $\hat{\Lambda}(i)$:
\begin{itemize}
    \item New cluster label, $\hat{k}$: $\hat{\Lambda}_k(i) = \Lambda_k(i)$ for all $k \in M_u(i)\cup\{-1\}$ and $\hat{\Lambda}_{\hat{k}}(i) = \beta_{\hat{k}}$ is initialized,
    \item Single cluster: $\hat{\Lambda}_k(i) = \Lambda_k(i)$ for all $k$,
    \item Merged cluster labels, $\hat{k} \in \hat{K}$: $\hat{\Lambda}_k(i) = \Lambda_k(i)$ for all $k \in (M_u(i)\cup\{-1\}) \setminus \hat{K}$ and $\hat{\Lambda}_{\tilde{k}}(i) = \sum_{\hat{k}} \Lambda_{\hat{k}}(i)$ is initialized accordingly.
\end{itemize}
Thus, the final update is written as
\begin{align*}
    p(\Lambda(i+1)\ |\ C_u^d[i+1], \beta) = \\ \frac{p_i(c_u^d(i+1)\ |\ \Lambda(i), \beta, C_u^d[i])\hat{p}(\Lambda(i)\ |\ C_u^d[i], \beta))}{p(c_u^d(i+1)\ |\ \beta, C_u^d[i])},
\end{align*}
Since the Dirichlet distribution is conjugate prior to the Categorical distribution, it holds that $p(\Lambda(i+1) | C_u^d[i+1], \beta) \sim \text{Dir}(K(i+1)+1, n(i+1) + \beta)$, where $n(i+1) = (n_{-1}(i+1), n_{0}(i+1), \ldots, n_{K(i+1)-1}(i+1))$ and $n_j(i+1)$ is the number of trips in $C_u^d[i+1]$ ending in $j$ up until timestep $i+1$.

A similar approach can be performed for the conditional distributions with $\lambda_j(i)$.
One will then end up with the update rule
\begin{align*}
    p(\lambda_j(i+1)\ |\ C_u^{\{j\}}[i+1], \alpha_j) = \\ \frac{p_i(C_u^{\{j\}}(i+1)\ |\ \lambda_j(i), \alpha_j, C_u^{\{j\}}[i])\hat{p}(\lambda_j(i)\ |\ C_u^{\{j\}}[i], \alpha_j))}{p(C_u^{\{j\}}(i+1)\ |\ \alpha_j, C_u^{\{j\}}[i])},
\end{align*}
where the notation $C_u^{\{j\}}$ is once again used to denote all the trips starting in $j$.
Again, it holds that $p(\lambda_j(i+1) |  C_u^{\{j\}}[i+1], \alpha_j) \sim \text{Dir}(K(i+1)+1, \hat{n}_j(t+1) + \alpha_j)$ due to conjugacy, where $\hat{n}_j(t+1) =  (n_{j,-1}(i+1), n_{j,0}(i+1), \ldots, n_{j,K(i+1)-1}(i+1))$ and $n_{jk}(i+1)$ is the number of transitions from $j$ to $k$ up until timestep $i+1$.
Finally, setting $\alpha_{jk} = 0$ would result in the maximum likelihood estimate of the parameters.

\subsection{Expert model}
Another option for the prediction model in the online framework is to use an expert model, as presented in \cite{cesa2006prediction}.
An expert model requires a set of experts and a reward function.
In our case, the action set corresponds to the set of possible destinations and the reward is $1$ if an expert makes the correct prediction, and $0$ otherwise.
More precisely, expert models, or learning with expert advice, is an online learning approach where the rewards in each timestep are known for all available actions.
In this section, we adapt this approach to our trip destination problem, where we address several issues such as a dynamic action set.

The athors in \cite{kleinberg2010regret} presents an algorithm called \textit{Follow the Awake Leader} (FTAL), which considers the expert setting with a dynamic set of available actions at every timestep.
It introduces the concept of sleeping experts, which means that the experts are allowed to sleep for some periods of time, i.e. they are not available at those specific time periods.
With some modifications we adapt it to our setup. We consider the following assumptions:
\begin{enumerate}
    \item There is an infinite number of sleeping experts,
    \item Once an expert wakes up, it will stay awake.
    \item Experts can merge.
\end{enumerate}

The modified algorithm is described in Alg. \ref{alg4}, where the action set $A_i$ would correspond to our cluster space $M_u(i)\cup\{-1\}$ and the different experts are the possible $k\in M_u(i)\cup\{-1\}$ options.
For each new trip, the actions played previously are first put in a set $A$ (line 3).
If $A$ is empty, then a random expert is played, and otherwise the expert with the highest average reward is selected (line 4-6).
The rewards are obtained for all available actions, and the stored parameters are updated (line 9-11).
Finally, the function $\textsc{UpdateActionSet}(A_{i-1})$ is used to update the set of available actions.
However, since $A_i$ would correspond to our state space $M_u(i)$, it is obtained as a result of the clustering.
\begin{algorithm}
\caption{Online expert model}
\textbf{Parameters: }{$A_i$ is the set of available action at time $i$, $z_k$ is the cumulative reward for action $k$, $n_k$ is the number of consecutive timesteps action $k$ has been available}
\begin{algorithmic}[1]
\label{alg4}
\STATE Initialize $A_0$, and $n_k = 0, z_k=0$ for all $k \in A_0$
\FOR{$i = 1, \ldots, N_u$}
    \STATE $A = \{k \in A_{i-1} : n_k > 0\}$
    \IF{$A = \emptyset$}
        \STATE Play random available expert $k$, or none
    \ELSE
        \STATE Play expert $k^* = \arg\max_{k\in A}(z_k/n_k)$
    \ENDIF
    \STATE Observe reward $R_k$ for all $k \in A_{i-1}$
    \STATE $z_k \leftarrow z_k + R_k$ for all $k \in A_{i-1}$
    \STATE $n_k \leftarrow n_k + 1$ for all $k \in A_{i-1}$
    \STATE $A_i \leftarrow \textsc{UpdateActionSet}(A_{i-1})$
\ENDFOR
\end{algorithmic}
\end{algorithm}
In this setup, each expert suggests performing a single specific action all the time, i.e. expert $k$ would always predict $k$.

Similar to the Bayesian model, there will be one expert model that considers all trips, as well as one expert model conditioned on each of the starting locations.
In fact, the main difference between these prediction models is the way that a destination is selected.
In the Bayesian approach, the average is taken over the total number of trips, whereas in the expert model it is instead taken over the number of trips for which the destination has been available.
One advantage that this approach have over the Bayesian approach is one can experiment with the definition of the rewards without changing the model itself.
It is also common for expert models to be accompanied with a regret bound, which is provided in \cite{kleinberg2010regret} for FTAL.
On the other hand, with the Bayesian approach it is possible to define priors if one has access to prior information.
There is also an intuitive way to include uncertainty in the predictions.

\subsection{Regret analysis}
A common way to investigate the performance of online learning methods is to look at the regret of the model as a function of the number of trips used for training.
Here, we define the regret in comparison to the offline model being trained on the entire trip history, and then evaluated on the very same data.
Let $p^*$ be the true discrete distribution conditioned on the source locations, and let $p_i$ be the corresponding predicted distribution at timestep $i$.
The squared Hellinger distance \cite{nikulin2001hellinger} between the true and the predicted distribution can then be defined as
\begin{align*}
    H^2(p^*,p_i)={\frac {1}{2}}{\sum _{x \in \mathcal{X}}\left(\sqrt{ p^*(x)}-\sqrt{p_i(x)}\right)^{2}}.
\end{align*}
However, this formulation assumes that both distributions are defined on the same probability space $\mathcal{X}$, which does not necessarily hold in our case.

Let $p_i$ be defined on $\mathcal{X}'_i$ and $p^*$ on $\mathcal{X}$ and assume that there is a surjective function $f:\mathcal{X} \to \mathcal{X}'_i$, i.e. $\forall x' \in \mathcal{X}'_i, \exists x \in \mathcal{X}$ such that $f(x) = x'$.
Further, let $f^c(x) = |\{x' \in \mathcal{X} | f(x') = f(x)\}|$ denote the number of elements $x' \in \mathcal{X}$ such that $f(x') = f(x)$, i.e. the number of elements in $\mathcal{X}$ that maps to $f(x) \in \mathcal{X}'_i$.
Then, one way to define the squared Hellinger distance between $p^*$ and $p_i$ is:
\begin{align*}
    H^2(p^*,p_i) :=\frac {1}{2}\sum_{\substack{x \in \mathcal{X}, \\ f^c(x) > 0}}
    \left(\sqrt{p^*(x)}-\sqrt{\frac{p_i(f(x))}{f^c(x)}}\right)^{2}
    + {\frac {1}{2}}{\sum _{\substack{x \in \mathcal{X}, \\ f^c(x) = 0}} p^*(x)},
\end{align*}
where the probability $p_i(x')$ is split equally amongst all $p^*(x)$ where $f(x) = x'$.
The first sum can be rewritten to yield the following formulation:
\begin{align*}
    H^2(p^*,p_i) :=\frac {1}{2}\sum_{x' \in \mathcal{X}'_i}
    \sum_{\substack{x \in \mathcal{X}, \\ f(x) = x'}}
    \left(\sqrt{p^*(x)}-\sqrt{\frac{p_i(x')}{f^c(x)}}\right)^{2}
    + {\frac {1}{2}}{\sum _{\substack{x \in \mathcal{X}, \\ f^c(x) = 0}} p^*(x)}.
\end{align*}
We consider this formulation from this point forward.

We split the metric into two sub-errors, $H^2_d(p^*,p_i)$ and $H^2_s(p^*,p_i)$, representing the distributional error and state-space error respectively.
First of all, let us define the distributional error:
\begin{align*}
    H^2_d(p^*,p_i) := \frac {1}{2}\sum_{x' \in \mathcal{X}'_i}
    \left(\sqrt{\sum_{\substack{x \in \mathcal{X}, \\ f(x) = x'}} p^*(x)}-\sqrt{\frac{p_i(x')}{f^c(x)}}\right)^{2},
\end{align*}
which essentially implies that $p_i(x')$ should be equal to the sum of $p^*(x)$ for all $x \in \mathcal{X}$ such that $f(x) = x'$.
The state-space error can then be implicitly defined as
\begin{align*}
     H^2_s (p^*,p_i) := H^2(p^*,p_i) - H^2_d(p^*,p_i).
\end{align*}
Note that this is only properly defined if $H^2(p^*,p_i) \geq H^2_d(p^*,p_i)$, which is not trivially true for the parts of the sum where $f^c(x) > 1$.
However, the following theorem shows that this indeed holds.

\begin{theorem}
Given that
\begin{enumerate}
    \item $\boldsymbol{p} = [p_1, \ldots, p_k]$ is the true distribution over a subset of $k$ different states,
    \item $\boldsymbol{q} = [q/k, \ldots, q/k]$ is the predicted distribution over the same $k$ states
\end{enumerate}
then
$H^2(\boldsymbol{p},\boldsymbol{q}) \geq H_d^2(\boldsymbol{p},\boldsymbol{q})$.
\end{theorem}

\begin{proof}
The overall squared Hellinger distance for these states are:
\begin{align}
\label{proof_eq1}
\begin{split}
    H^2(\boldsymbol{p},\boldsymbol{q})&={\frac {1}{2}}{\sum _{j=1}^{k}(\sqrt{ p_j}-\sqrt{q/k})^{2}} \\
    &= \frac {1}{2}\sum_{j=1}^{k} \left(p_j + \frac{q}{k}\right) - \sum_{j=1}^{k} \sqrt{\frac{p_j q}{k}}.
    \end{split}
\end{align}
The distribution error can be written as:
\begin{align}
\label{proof_eq2}
    \begin{split}
    H_d^2(\boldsymbol{p},\boldsymbol{q}) &= \frac {1}{2}\left(\sqrt{\sum_{j=1}^{k} p_j} - \sqrt{q}\right)^2 \\
    &= \frac {1}{2}\sum_{j=1}^{k} \left(p_j + \frac{q}{k}\right) - \sqrt{q\sum_{j=1}^{k} p_j}.
    \end{split}
\end{align}

Now, we show from Eq. \ref{proof_eq1} and \ref{proof_eq2} that $H^2(\boldsymbol{p},\boldsymbol{q}) \geq H_d^2(\boldsymbol{p},\boldsymbol{q})$, i.e. after some simplifications we show that:
\begin{align*}
    - \sum_{j=1}^{k} \sqrt{\frac{p_j q}{k}} \geq - \sqrt{q\sum_{j=1}^{k} p_j} \iff \\
    \sum_{j=1}^{k} \sqrt{\frac{p_j}{k}} \leq \sqrt{\sum_{j=1}^{k} p_j} \iff \\
    \left(\sum_{j=1}^{k} \sqrt{p_j} \right)^2 \leq k\sum_{j=1}^{k} p_j
\end{align*}
The left hand side of the last inequality can be rewritten as
\begin{align*}
    \left(\sum_{j=1}^{k} \sqrt{p_j} \right)^2 = \left(\sum_{j=1, j \neq i}^{k} \sqrt{p_j} \right)^2 + p_i + \sum_{j=1, j \neq i}^{k} 2\sqrt{p_ip_j},
\end{align*}
and the right hand side can be written as
\begin{align*}
   k\sum_{j=1}^{k} p_j = (k-1)\sum_{j=1, j \neq i}^{k} p_j + \sum_{j=1, j \neq i}^{k} p_j + kp_i
\end{align*}

Looking only at the first terms in these expression, one notices that it is a scaled down form of the original problem.
Thus, if we can show that
\begin{align*}
    p_i + \sum_{j=1, j \neq i}^{k} 2\sqrt{p_ip_j} \leq \sum_{j=1, j \neq i}^{k} p_j + kp_i
\end{align*}
we have proven the claim.
This inequality can be simplified accordingly:
\begin{align*}
    p_i + \sum_{j=1, j \neq i}^{k}2 \sqrt{p_ip_j} \leq \sum_{j=1, j \neq i}^{k} p_j + kp_i \iff \\
    \sum_{j=1, j \neq i}^{k} (p_i + p_j - 2\sqrt{p_ip_j}) \geq 0 \iff \\
    \sum_{j=1, j \neq i}^{k} (\sqrt{p_j} - \sqrt{p_i})^2 \geq 0
\end{align*}
where the last inequality is trivially true.
Thus, we concluded that $H^2(\boldsymbol{p},\boldsymbol{q}) \geq H_d^2(\boldsymbol{p},\boldsymbol{q})$. \qed
\end{proof}

Thus, the Hellinger regret can finally be defined as
\begin{align*}
    \textsc{regret}(i) &:= \sum_{i' \leq i} H^2(p^*,p_{i'})  \\
    &= \sum_{i' \leq i} H^2_d(p^*,p_{i'}) + H^2_s(p^*,p_{i'}),
\end{align*}
i.e. the cumulative squared Hellinger distance.

\section{Experiments}
In this section, we investigate and evaluate the online trip prediction framework, i.e. both the clustering technique and the prediction model, on a real-world dataset of private vehicle trip histories.
The clustering is mainly evaluated using the well known cluster metrics, in order to understand how the clusters of the online clustering evolve as more trips are observed.
The online prediction models, and the full online framework, are instead evaluated in comparison to the offline pipeline using the accuracy on a held out test set.
Furthermore, they are evaluated using a novel regret metric based on the Hellinger distance, which evaluates the similarity between the predicted and the true distribution of destinations.

\subsection{Data}
In this paper, we use the real-world data collected in \cite{karlsson2013swedish}.
It consists of over 700 GPS-tracked vehicles, or devices, registered either in the county of Västra Götaland or in Kungsbacka municipality.
These are located in the south-western part of Sweden and include Gothenburg, which is the second largest city of the country.

A dataset consisting of trips for each of the vehicles is extracted from the original GPS-logs.
This was done by defining the end of a trip as the loss of GPS fixation, i.e., when the vehicle has been turned off.
In addition, a vehicle speed of less than 0.1 km/h for 10 minutes was also used to signify the end of a trip.
Finally, two consecutive trips have been merged if the time between them is less than 10 seconds.

\begin{figure}[ht]
    \centering
    \includegraphics[width=0.5\columnwidth]{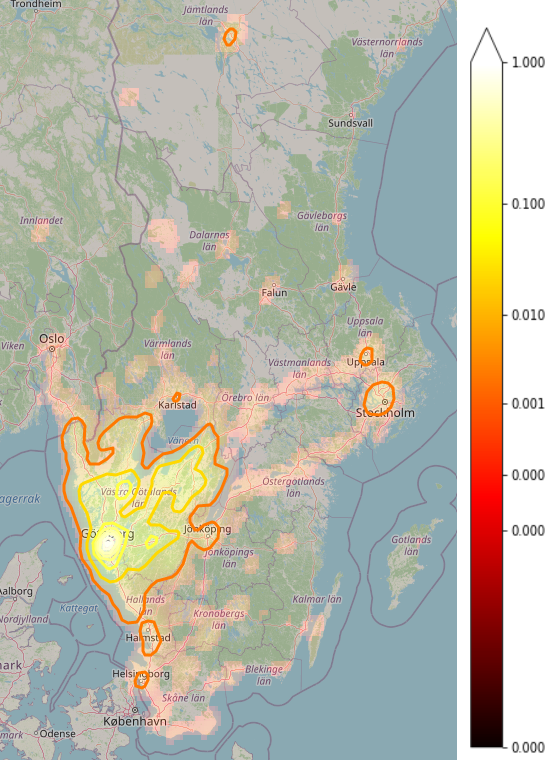}
    \caption{Gaussian kernel density (followed by a power-law normalization of its linear mapping to the 0-1 range) of trip destinations for all vehicles in \cite{karlsson2013swedish} after preprocessing.}
    \label{fig:density_plot}
\end{figure}

In order to be somewhat consistent with the data processing performed in similar works \cite{ashbrook2003using,alvarez2010trip,zong2019trip}, we perform additional filtering to the data provided in \cite{karlsson2013swedish}:
\begin{enumerate}
  \item[i] Trips being shorter than 100 meters, or less than 4 minutes, are discarded.
  \item[ii] Vehicles with a trip history shorter than 30 days, or with a frequency of less than 1 trip per day, are discarded (guarantees at least 30 distinct trips per user).
\end{enumerate}

\begin{figure}[ht]
    \centering
    \includegraphics[width=1\columnwidth]{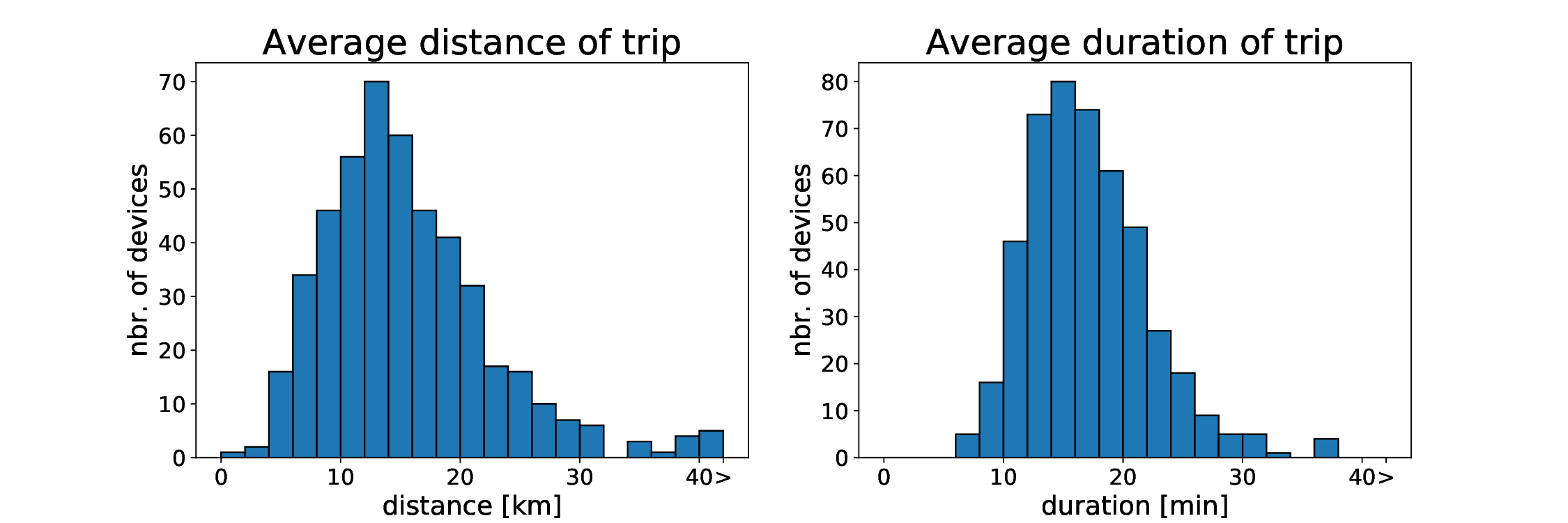}
    \caption{Average distance (left) and average duration (right) of the trips considered in the dataset.}
    \label{fig:data_info}
\end{figure}

After the preprocessing, about 55\% out of the original trips and 66\% out of the vehicles remain (74453 out of 134756, and 473 out of 716, respectively).
The remaining trip destinations can be seen in Fig. \ref{fig:density_plot}, which shows the power-law normalization of the values of a Gaussian kernel density estimation after they have been linearly mapped to the range $[0,1]$.

We observe that the majority of the trips end in the south-western part of Sweden, i.e., where the vehicles are also registered.
Furthermore, in Fig. \ref{fig:data_info} the average distance and duration of the remaining trips are shown for each of the users.
Across all users, the average length of a trip is $15.57$ km and the average duration is $17.05$ min.
In this data, 413 unique private vehicles are studied.
Each of these vehicles corresponds to one separate case study, which means that we effectively study 413 different cases.
For this type of data this is substantial, since most of the existing public datasets are not private vehicle driving histories, but correspond to taxis, public transport, etc.

\subsection{Clustering}
By clustering the entire dataset with DBSCAN using the parameters $\epsilon = 100$ m and $m = 2$, one can interpret each cluster as a location that the user has visited at least two times.
Fig. \ref{fig:offline_cluster_stats} illustrates the number of clusters and the percentage of trips ending in a cluster for all the users in the dataset.
On average, the number of found clusters, i.e. $K$, is $15.5$ and $75.4\%$ of users trips end in this set of clusters.

\begin{figure}[ht]
    \centering
    \includegraphics[width=1\columnwidth]{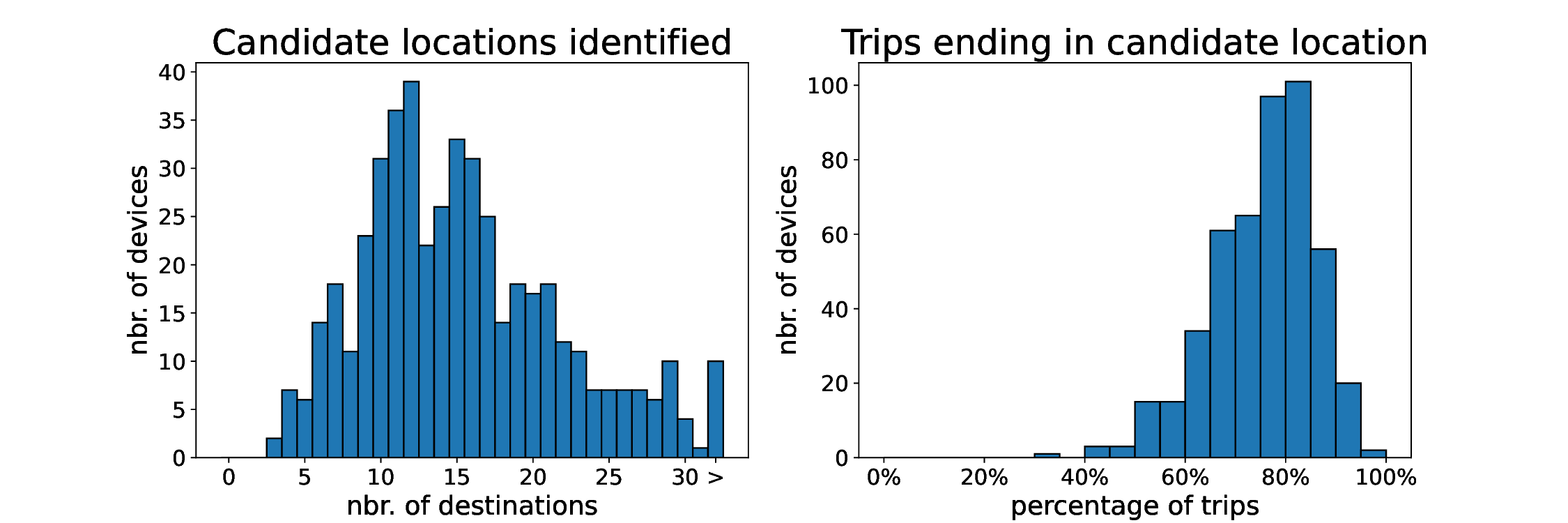}
    \caption{Offline clustering: Number of clusters found (left) and the percentage of trips ending in a cluster (right) for the different users in the dataset.}
    \label{fig:offline_cluster_stats}
\end{figure}

The same experiment is performed using the two variants of the online clustering algorithm.
Fig. \ref{fig:online_cluster_stats_v1} and Fig. \ref{fig:online_cluster_stats_v2} show the results for variant 1 and variant 2 respectively.
The parameters $m$ and $\epsilon$ are the same as for the offline clustering.
The additional parameter $r$ in variant 1 is set to $1/2$, and the function $\textsc{DeleteOldPoints}(t)$ is adjusted to remove the points older than 28 days in both variants.
Other combinations of the additional parameters were tested, but the results are robust to the changes, and they are not affected much unless the parameters are modified drastically.
The average number of clusters found and trips ending in the set of clusters is $13.6$ and $64.1\%$ for variant 1 and $13.2$ and $64.7 \%$ for variant 2.

\begin{figure}[ht]
    \centering
    \includegraphics[width=1\columnwidth]{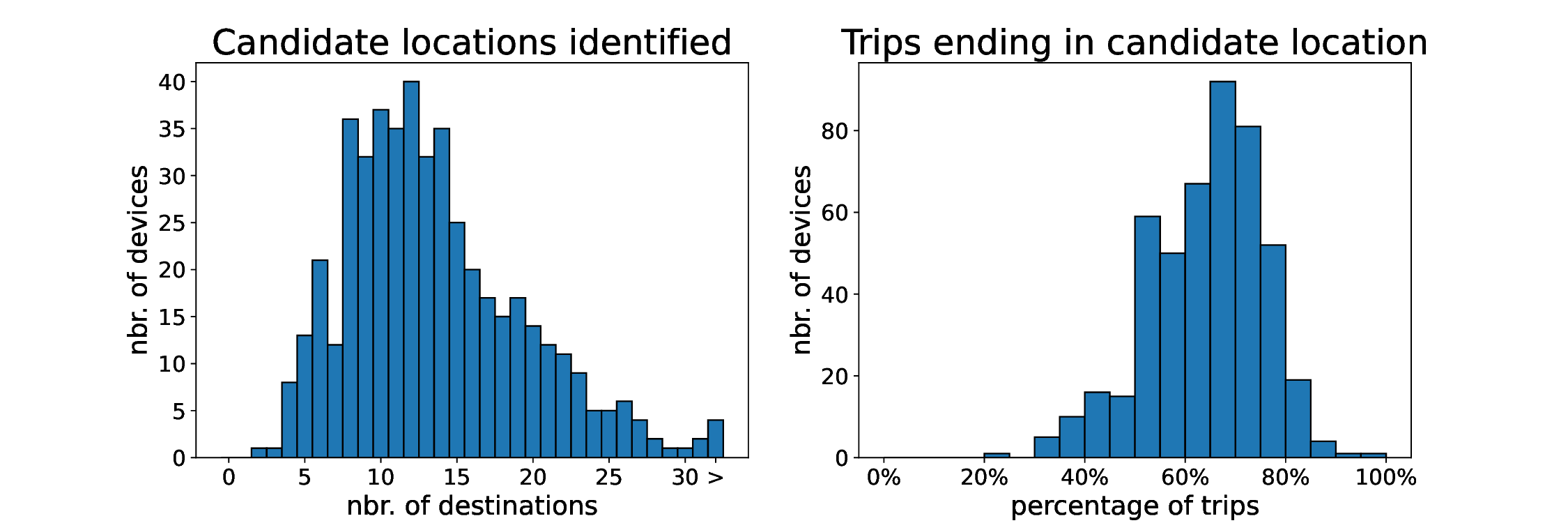}
    \caption{Online clustering - V1: Number of clusters found (left) and the percentage of trips ending in a cluster (right) for the different users in the dataset.}
    \label{fig:online_cluster_stats_v1}
\end{figure}

\begin{figure}[ht]
    \centering
    \includegraphics[width=1\columnwidth]{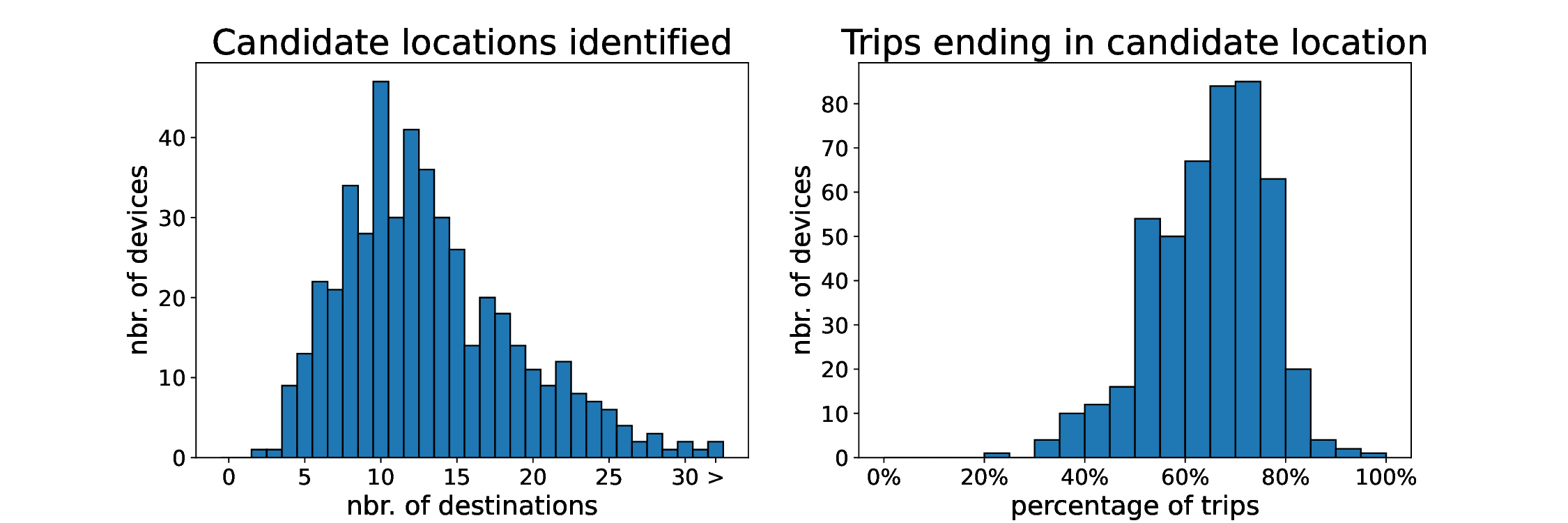}
    \caption{Online clustering - V2: Number of clusters found (left) and the percentage of trips ending in a cluster (right) for the different users in the dataset.}
    \label{fig:online_cluster_stats_v2}
\end{figure}

Looking at the difference between the offline and online clustering methods, we notice that the number of clusters as well as the percentage of trips ending in a cluster appear to have dropped, although this can be anticipated.
The number of clusters are affected by the fact that non-clustered points are only kept for a given number of days.
Furthermore, by storing the clusters as centroids, the possibility of merging two nearby clusters increases.
The percentage of trips ending in a cluster is affected by the number of clusters, but perhaps even more by the fact that the assignment of labels is done in an online way.
In other words, the first time a place is visited it cannot yet be labeled as a candidate location and will at that point in time be considered an outlier.
This means that the first point to appear where a cluster is going to be formed will never be counted.

Another way that one can evaluate the online clustering algorithms is to look at the evolution of the rand score, mutual information, and v-measure score \cite{hubert1985comparing,vinh2010information,rosenberg2007v}.
All these metrics compare the predicted cluster labels of each trip with the true labels and yield a score that is upper limited by $1$, where a score of $1$ indicates a perfect match.
The true labels in this comparison are those obtained from the offline clustering algorithm, when run on the full dataset, i.e. on all available trips for each user.

In Figs. \ref{fig:cluster_metrics_v1} and \ref{fig:cluster_metrics_v2}, one can see the performance using these metrics for the two variants of the online clustering algorithms.
In general, we observe that the average of all metrics appears to increase as more trips are processed, as expected.
Once again, it is important to emphasize that the cluster labels of the online clustering algorithms are produced in an online manner.

\begin{figure}[ht]
    \centering
    \includegraphics[width=1.1\linewidth]{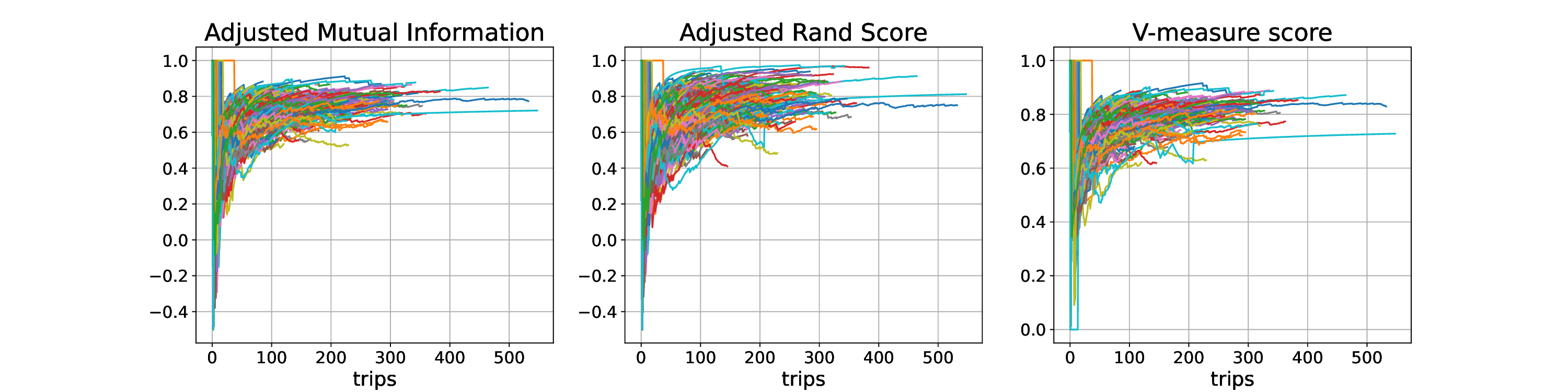}
    \caption{Evolution of the mutual information, the rand score and the v-measure score using the first variant of the online clustering algorithm.}
    \label{fig:cluster_metrics_v1}
\end{figure}

\begin{figure}[ht]
    \centering
    \includegraphics[width=1.1\columnwidth]{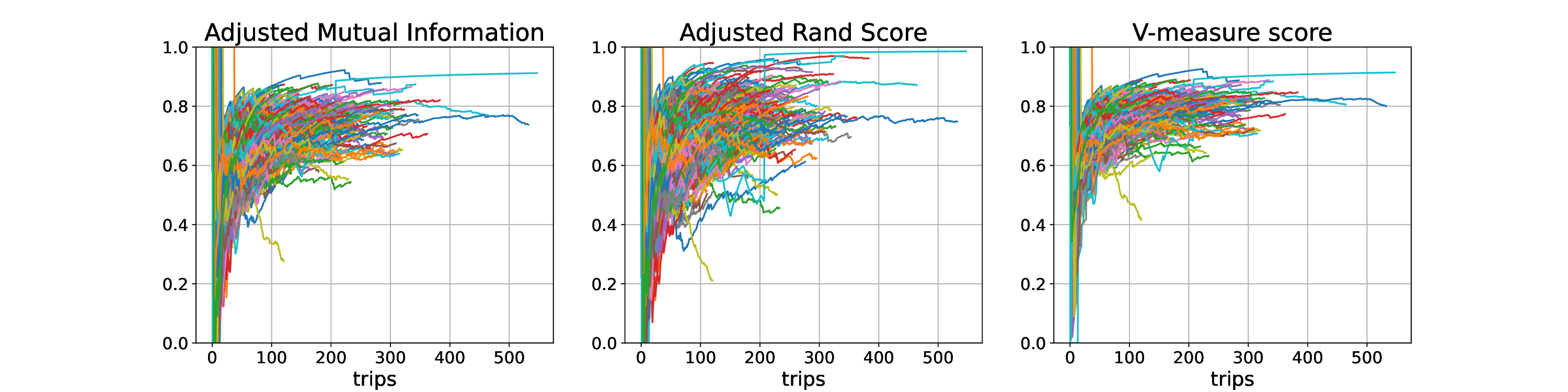}
    \caption{Evolution of the mutual information, the rand score and the v-measure score using the second variant of the online clustering algorithm.}
    \label{fig:cluster_metrics_v2}
\end{figure}

Another interesting aspect is the similarity between the true clusters and those found by the online clustering algorithms after considering the full trip history.
This can be done by computing the labels of the online clustering algorithm after it has been trained on the full trip history.
The histogram in Fig. \ref{fig:cluster_metrics_hist} shows the difference between the two clustering variants.
Using the same metrics, one can see that the clusters obtained from the first variant are more similar to the true ones, since the first variant consistently yields higherscores than the second variant.

\begin{figure}[ht]
    \centering
    \includegraphics[width=1\columnwidth]{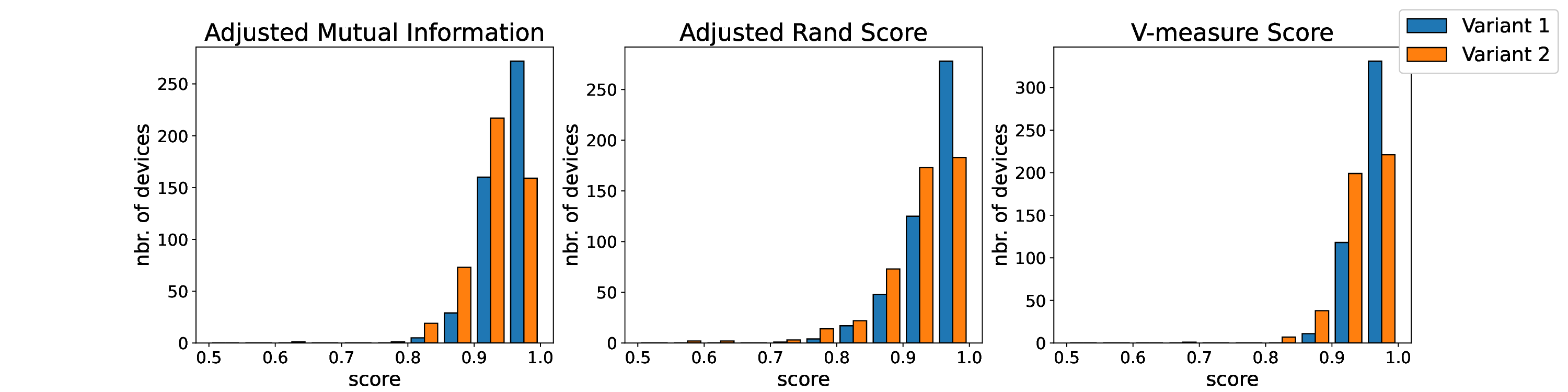}
    \caption{Comparison of the clusters obtained from the two clustering variants after considering the full trip history.}
    \label{fig:cluster_metrics_hist}
\end{figure}

On average, the first variant yields the scores $[0.953, 0.947, 0.963]$ for the mutual information, rand score and v-measure score, respectively.
The second variant gives the average scores $[0.931, 0.925 , 0.945]$.
Thus, it seems that on average the clusters that are found by the first variant are more similar to the true clusters than the second variant, albeit only slightly.
Interestingly, if one instead looks at the minimum values, i.e. the worst performance amongst the different users, the first variant gives $[0.835, 0.739, 0.869]$, whereas the second variant yields $[0.631, 0.452, 0.687]$.
Thus, in the worst case scenario, the clusters obtained from the first variant appears to be more stable as well.
This could partly be attributed to the fact that the first Variant has the possibility to shape the clusters as a union of centroids, i.e., they are not necessarily circular.

\subsection{Evaluation of the entire framework}
To investigate the full framework, the first $80\%$ of each users' trip history is used to create a training set, leaving the rest for testing.
A clustering algorithm is run to produce $C_u^s$ and $C_u^d$ for the training set, and the transitions are used to estimate the parameters of the distributions.

\subsubsection{Offline setting}
If the starting location is an outlier, i.e. $c_u^s = -1$, the distribution $p(c_u^d)$ is used to predict the destination by $k^* = \arg\max_{k\in M_u} p(c_u^d = k)$.
If $c_u^s$ is not an outlier, i.e. $c_u^s = j$ where $j\in M_u$, the prediction is done as $k^* = \arg\max_{k\in M_u} p(c_u^d = k\ |\ c_u^s = j)$.
In other words, the prediction always corresponds to the cluster with the highest probability, excluding the outliers.

By evaluating the accuracy of the predictions, we find that out of all trips, the proposed model is able to predict the next destination in $36.15\%$ of the cases on average.
Looking only at the trips that end in one of the clusters, i.e. those that can actually be predicted, the accuracy increases to $56.22\%$ on average.
The distributions of the accuracy over the different users in both cases are displayed in Fig. \ref{fig:offline_acc}.

\begin{figure}[ht]
    \centering
    \includegraphics[width=1\columnwidth]{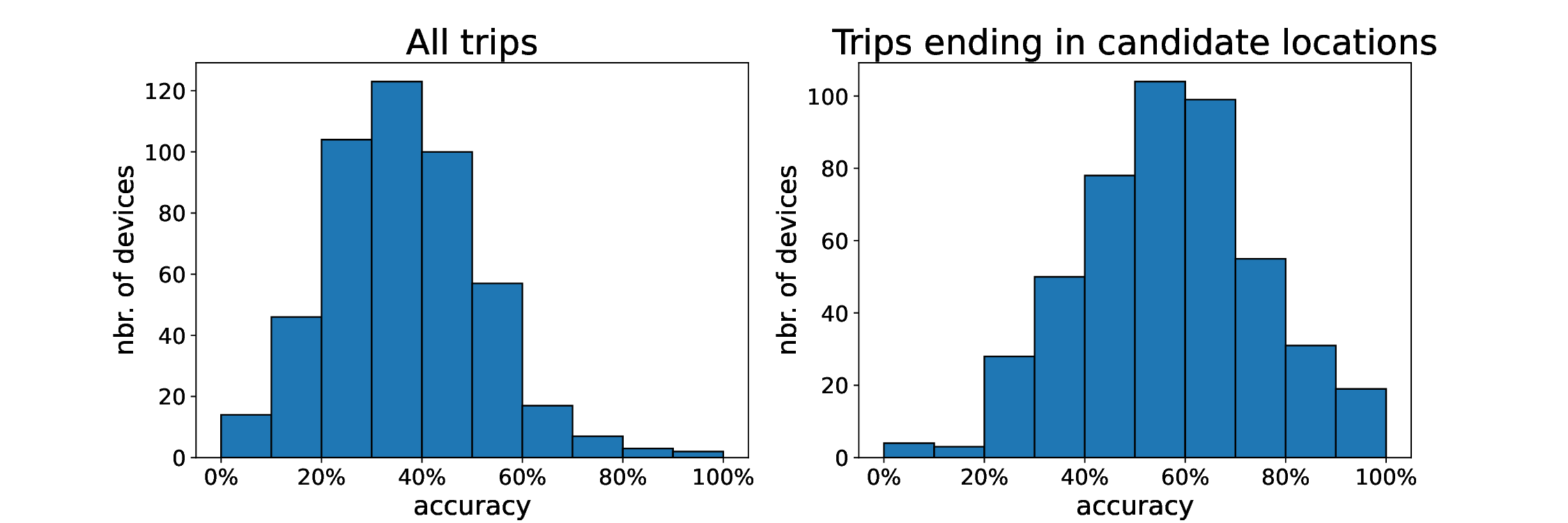}
    \caption{Offline setting: Prediction accuracy of the different users if considering all trips (left) and only trips ending in candidate locations (right).}
    \label{fig:offline_acc}
\end{figure}

\subsubsection{Online setting}
We evaluate the online setting, including both the clustering and the prediction model, on the same data as in the offline case, i.e. when testing on the last $20\%$ of each users trip history.
This yields similar results to the offline setting.
For the Bayesian model the prediction is always made as the cluster with the highest probability, excluding the outliers.
The source, $c_u^s(i)$, is predicted using Alg. \ref{alg:source_cluster} with $\delta = 2$.
If the source $c_u^s(i) = -1$, i.e. it is an outlier, the distribution of $p_i(c_u^d)$ is used to predict $k^* = \arg\max_{k\in M_u(i)} p_i(c_u^d = k)$.
Instead, if $c_u^s(i)$ is not an outlier, i.e. $c_u^s(i)=j$ with $j \in M_u(i)$, the prediction is made according to  $k^* = \arg\max_{k\in M_u(i)} p_i(c_u^d = k\ |\ c_u^s = j)$.
Similarly, the expert model also excludes outliers in the prediction, by not considering the outliers in the set of available actions when selecting an expert.
The expert model corresponding to $p_i(c_u^d)$ is used when $c_u^s(i) = -1$, and the models conditioned on the starting locations are used otherwise.

\begin{figure}[hb]
    \centering
    \includegraphics[width=0.9\columnwidth]{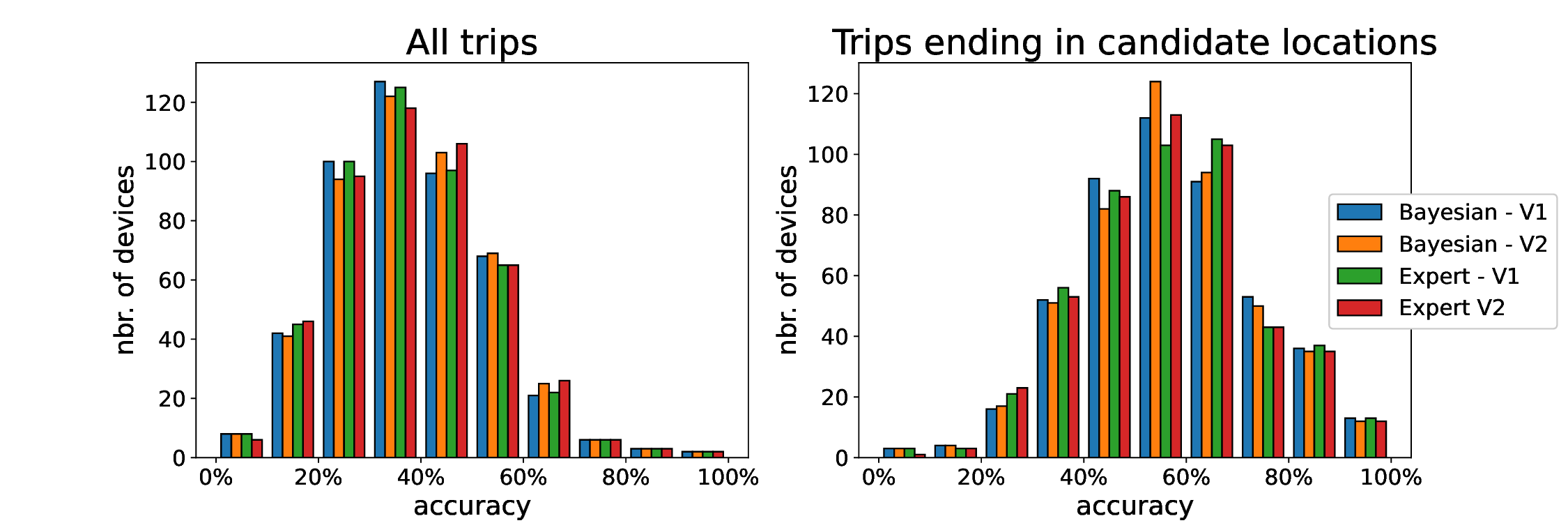}
    \caption{Online setting: Prediction accuracy of the different users if considering all trips (left) and only trips ending in candidate locations (right) for the proposed online prediction models.}
    \label{fig:online_acc}
\end{figure}

In Fig. \ref{fig:online_acc}, the distribution of the accuracy is shown for both prediction models, as well as the two clustering variants.
In general, upon visual inspection the different configurations appear to perform equally well.
Furthermore, the accuracy for all trips, as well as for the subset of trips ending in a candidate location, look similar to the histograms presented for the offline setting.
The Bayesian model using the first clustering variant has an average accuracy of $37.51\%$ and $56.05\%$ for all trips and trips ending in candidate locations, respectively.
Instead, using the Bayesian model with the second clustering variant one obtains $38.26\%$ and $56.12\%$ for the two cases.
Finally, using the expert prediction model yields $37.22\%$ and $55.61\%$ with the first clustering variant, and $38.05\%$ and $55.84\%$ with the second clustering variant.
Regardless of prediction configuration, these results are comparable to the offline version, with only minor deviations.

\subsubsection{Regret}
The regret for several clusters with a large number of trips is shown in Fig. \ref{fig:regret} for both the Bayesian model and the Expert algorithm using the two clustering variants.
The regret is split into the state-space error and the distribution error.
For comparison, we consider three baselines:

\begin{figure}[htb]
    \centering
    \begin{subfigure}[t]{1\columnwidth}
        \centering
        \includegraphics[width=1\textwidth]{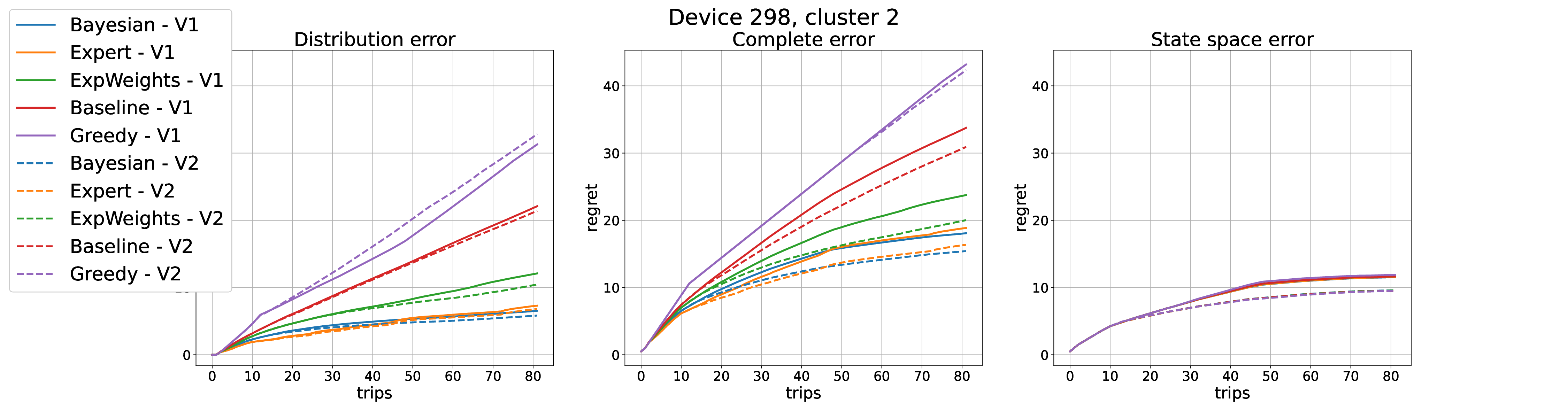}
    \end{subfigure}%

    \begin{subfigure}[t]{1\columnwidth}
        \centering
        \includegraphics[width=1\textwidth]{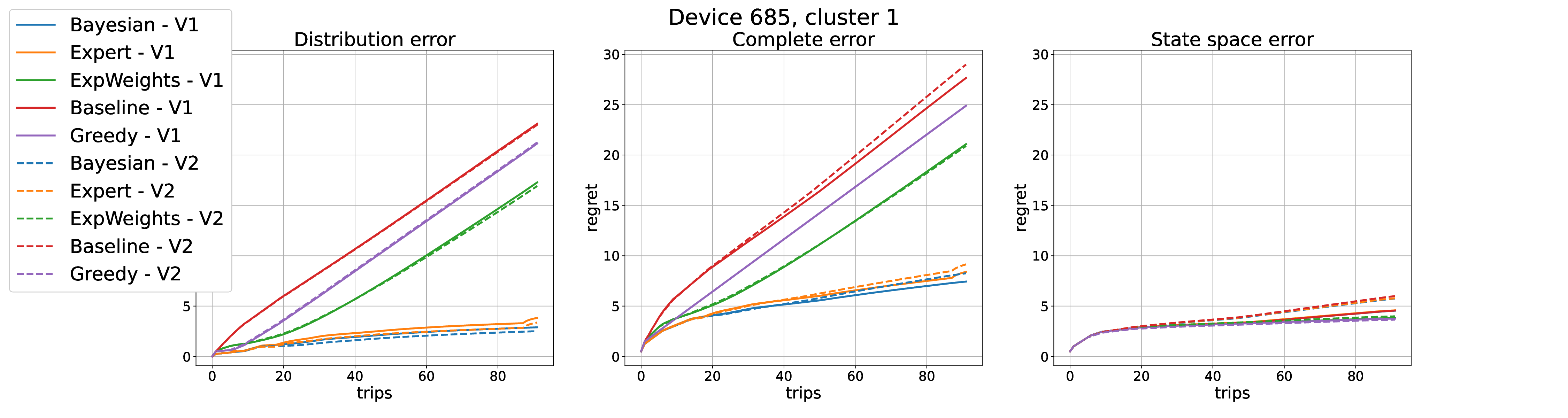}
    \end{subfigure}

    \begin{subfigure}[t]{1\columnwidth}
        \centering
        \includegraphics[width=1\textwidth]{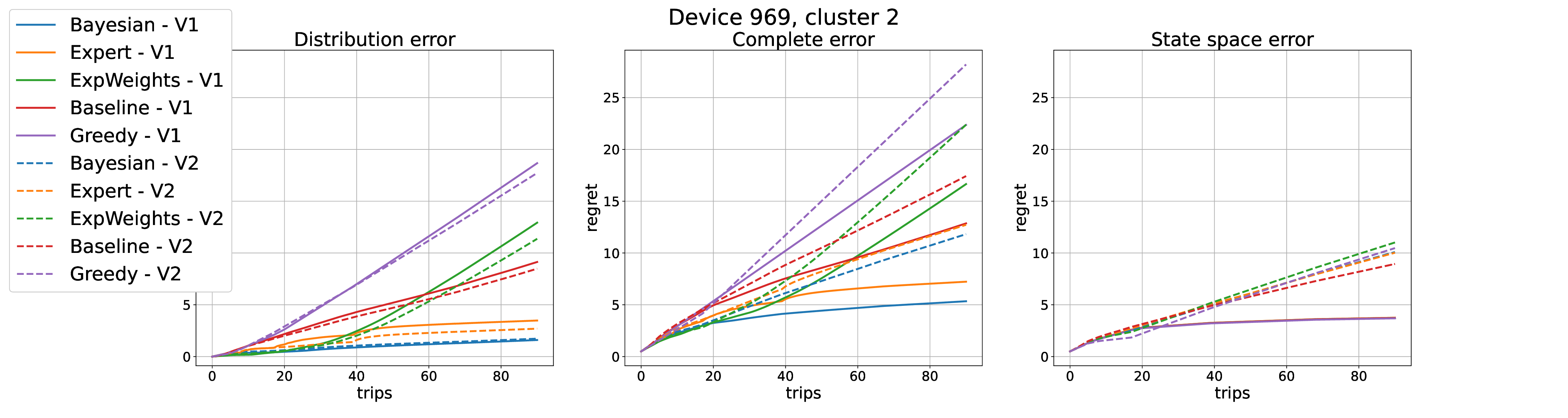}
    \end{subfigure}
    \caption{Hellinger regret for the Bayesian method and the expert algorithm, compared to three baselines: Non-conditioned distribution, Exponential Weights algorithm, and a Greedy algorithm.}
    \label{fig:regret}
\end{figure}

\begin{figure}[hbt]
    \centering
    \includegraphics[width=1\columnwidth]{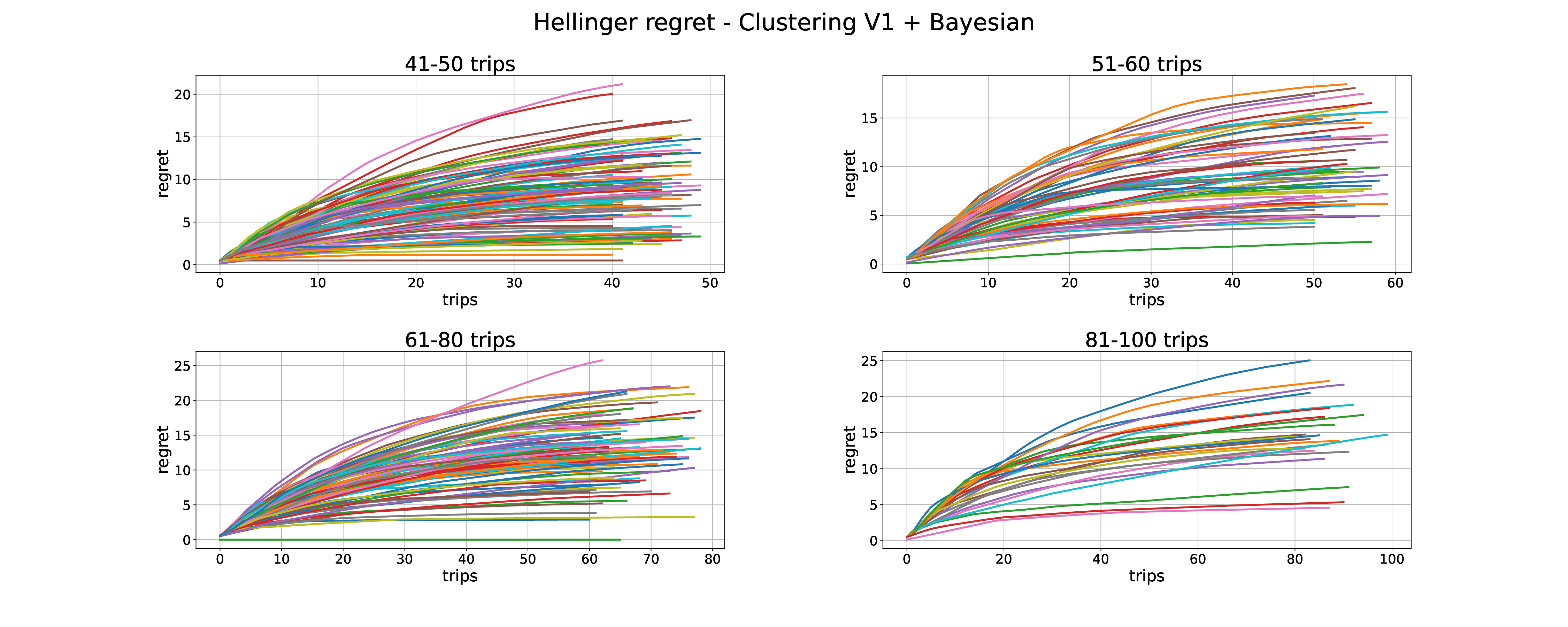}
    \caption{Hellinger regret for the Bayesian model with the first clustering variant.}
    \label{fig:regret_state_v1_bayes}
\end{figure}

\begin{itemize}
    \item[-] Baseline, where we only use the distribution of $p(c_u^d)$, i.e. never conditioning on the source of a trip.
    \item[-] ExpWeights, which is the exponential weights algorithm \cite{cesa2006prediction}.
    \item[-] Greedy, that is a greedy variant of the Bayesian model, i.e. we select the most likely options with $100 \%$ confidence.
\end{itemize}

It is worth to emphasize that no established baseline exists for this specific problem setting.
The baselines that do exist for trip destination prediction do not consider online learning, or even the creation of the state space (clusters), in their model.
Looking at these three examples, we observe that the Bayesian model and the expert algorithm outperform the alternative methods.

We also observe the impact of the online clustering variant on the performance.
In the first example (device 298, cluster 2) the second clustering variant helps to decrease the distribution error, without increasing the state-space error.
The reason is that this variant assigns most of the early trips to the different clusters, since the clusters generated by this variant cover a larger area.
However, this could also be a disadvantage, which is the case for the last example (device 948, cluster 2).
In this example, two clusters are merged, which punishes the state-space error heavily.
In general, most cases behave similar to the second example (device 685, cluster 1), where the two clustering variants yield a similar performance.

Lastly, in Fig. \ref{fig:regret} we also observe a sub linear behaviour for the complete error in all the three examples, which indicates that learning improves over time.
This is true for essentially all examples of devices/clusters that have a sufficient trip history.
As an example, in Fig. \ref{fig:regret_state_v1_bayes} one can see the Hellinger regret for the Bayesian model with the first clustering variant for all devices/clusters with 41-100 trips.
The sub linearity can clearly be observed for all examples.
Such a behavior is usually expected from a proper online learning paradigm.

\section{Conclusion}
In this work, we developed a unified online framework for trip destination prediction consisting of (i) clustering and (ii) prediction model.
The online prediction models are generic and can easily be adapted to other offline or online prediction models in the Bayesian settings that has been studied, e.g., conditioned on additional attributes.

Firstly, we proposed two novel online clustering algorithms and two different online prediction models.
The clustering algorithms are online adaptation of the offline DBSCAN method, where the clusters are stored as centroids.
The first prediction models is an online adaptation of a Bayesian model conditioned on the starting position, whereas the second option is an adaption of an expert algorithm.

Secondly, we evaluated the online clustering algorithms and the full online framework on a real world trip dataset.
The clustering methods were shown to find the most important clusters of the offline solution.
We also demonstrated that the full framework yields consistent results with the offline model on unseen data.

Finally, we introduced a new evaluation metric suitable for the online framework.
This metric is able to distinguish between distributional error and state-space error, i.e., to distinguish between the clustering and prediction errors.
With sufficient trip histories we were able to show that the proposed methods converge to a probability distribution resembling the true underlying distribution with a lower regret compared to the baselines.

Future work will  consist of adding side information to the prediction models, such as time of day, month of year, weather, calendar information, etc.
Perhaps the simplest way to accomplish this is to condition the models on the additional attributes.
Another future extension is to investigate mixture of offline and online models and study how the performance improves if we already have an initial set of trips and a model trained based on them.
Finally, since the clusters represent the geographical locations, then it would be interesting to investigate a hierarchical variant of DBSCAN or other hierarchical clustering methods \cite{Chehreghani21,ChehreghaniAC08} in order ro take the cluster proximities into account.

\section*{Acknowledgement}
This work is funded by the Strategic Vehicle Research and
Innovation Programme (FFI) of Sweden. Project number: 49122-1.
We would like to thank Shafiq Urréhman, Niklas Legnedahl, Sri Vishnu, and Petter Frejinger from CEVT (China Euro Vehicle Technology AB) for insightful discussions and support.

\bibliographystyle{plain}      
\bibliography{references}   

\end{document}